\documentclass{article}

\usepackage{arxiv}
\usepackage[utf8]{inputenc}
\usepackage[T1]{fontenc}
\usepackage{hyperref}
\usepackage{url}
\usepackage{booktabs}
\usepackage{amsmath, amsfonts, amssymb, amsthm}
\usepackage{nicefrac}
\usepackage{microtype}
\usepackage{lipsum}
\usepackage{graphicx}
\usepackage{color}
\usepackage[symbol]{footmisc}
\usepackage{bbm}
\usepackage{subcaption}
\usepackage[ruled]{algorithm2e}
\usepackage{rotating}
\usepackage{multirow}
\usepackage{doi}
\usepackage[style=numeric,minbibnames=3,natbib=true]{biblatex}

\newcommand{\R}{\mathbb{R}}

\newtheorem{theorem}{Theorem}

\newtheorem{proposition}{Proposition}
\theoremstyle{definition}
\newtheorem{definition}{Definition}
\theoremstyle{remark}

\theoremstyle{plain}

\DeclareMathOperator{\sgn}{sgn}

\hypersetup{
    colorlinks,
    linkcolor= red,
    citecolor= cyan,
    urlcolor= blue
}

\title{A novel approach to graph distinction through GENEOs and permutants}

\author{
  Giovanni Bocchi \\
  Department of Environmental Science and Policy\\
  University of Milan\\
  Milan, MI 20133 \\
  \texttt{giovanni.bocchi1@unimi.it} \\
  \And
  Massimo Ferri \\
  Department of Mathematics and ARCES\\
  University of Bologna\\
  Bologna, BO 40126 \\
  \texttt{massimo.ferri@unibo.it} \\
  \And
  Patrizio Frosini \\
  Department of Mathematics and ARCES\\
  University of Bologna\\
  Bologna, BO 40126 \\
  \texttt{patrizio.frosini@unibo.it} \\
}

\begin{document}

\maketitle

\begin{abstract}
The theory of Group Equivariant Non-Expansive Operators (GENEOs) was initially developed in Topological Data Analysis for the geometric approximation of data observers, including their invariances and symmetries. This paper departs from that line of research and explores the use of GENEOs for distinguishing $r$-regular graphs up to isomorphisms. In doing so, we aim to test the capabilities and flexibility of these operators. Our experiments show that GENEOs offer a good compromise between efficiency and computational cost in comparing $r$-regular graphs, while their actions on data are easily interpretable. This supports the idea that GENEOs could be a general-purpose approach to discriminative problems in Machine Learning when some structural information about data and observers is explicitly given. 
\end{abstract}

\section*{Introduction}
Explainable machine learning models have recently emerged as an important part of the research in artificial intelligence and aim at devising methods and techniques that are understandable for humans~\cite{Ru19,Carrieri2021,Hicks2021}.
In this field, the use of concepts from topology and geometry has enabled developments that promise to make machine learning more easily interpretable, as required in many critical applications, where security and reliability are crucial elements. 
The research about group equivariant non-expansive operators (GENEOs) fits into this scientific context, offering the possibility of building small networks of operators that process the available information in a transparent and easily controllable way~\cite{BeFrGiQu19,Mi23,Boeal24}.

GENEOs have their roots in Topological Data Analysis and make available a mathematical theory for the approximation of observers, including their symmetries and shifting the attention from the data alone to the pairs \emph{(data, observer)}, seen as the main object of study.
This change of perspective is justified by the fact that in many applications, the interest is not directly focused on data, but on approximating the experts' behavior in the presence of some given information~\cite{Fr16}.
It is indeed well known that different agents can react in completely different ways to the presence of the same data, and this implies that data comparison cannot be separated from the problem of understanding observers' characteristics and preferences. 
We stress that an approach to data comparison based on geometrical operators offers various advantages. First, it allows us to translate the study of the ``shape of data'' into a precise mathematical problem concerning the approximation of observers. Secondly, it benefits from sound techniques and procedures that come from TDA, the area in which the GENEO theory was born.
Thirdly, it makes available a modular setting, and a great reduction in the number of components and parameters to be managed.

In the theory of GENEOs, data are represented by real-valued or vector-valued functions. GENEOs transform data belonging to a given functional space $\Phi$ into data belonging to another functional space $\Psi$, and their definition is based on two mathematical properties: these operators are equivariant for the action of a given transformation group $G$ and do not increase the distance between data. The first property means that GENEOs commute with each transformation in $G$. As a simple example, we can think of the blurring operator, acting on grey-level images seen as functions from $\R^2$ to $[0,1]$: if the blurring is performed by a convolution using a rotationally symmetric kernel $h:\R^2\to \R$ with $\int_{\R^2} |h(p)|\ dp\le 1$, and $G$ is the group of isometries of the real plane, then 
the blurring commutes with every transformation $g\in G$. This means that GENEOs respect data symmetries. 
We remind here that equivariance has taken a key role in the research about Deep Learning, making explicit the use of symmetries and consequently reducing the number of parameters required in the learning process~\cite{Ma12,BeCoVi13,ZVERP15,Ma16,AnRoPo16,cohen2016group,worrall2017harmonic,AERP19,gerken2023}.
The second property concerns the metric relation between the input and output of the operator: it guarantees that the distance between the input data is greater or equal to the distance between the output data. In other words, GENEOs simplify the information, in a metric sense.

Each GENEO space $\mathcal{F}$ benefits from good mathematical properties, such as compactness and convexity, when suitable assumptions are made on the space of data and appropriate topologies are chosen~\cite{BeFrGiQu19}. In particular, compactness guarantees that for every $\varepsilon>0$, we can find a finite set
$\mathcal{F}'=\{F_1,\ldots,F_s\}\subseteq \mathcal{F}$ such that any operator in $\mathcal{F}$ has a distance less than $\varepsilon$
from at least one $F_i$. In other words, GENEO spaces can be finitely approximated with arbitrary precision.

Since graphs can be represented as adjacency matrices and adjacency matrices can be interpreted as functions, GENEOs can be applied to compare graphs. In this framework, a technique for building GENEOs based on the concept of permutant has been recently extended to graphs~\cite{genperm2023}. This concept can be introduced when two data spaces $\Phi$, $\Psi$ and two transformation groups $G,K$ are given, respectively acting on the domain $X$ of the functions in $\Phi$ and the domain $Y$ of the functions in $\Psi$, together with a group homomorphism $T:G\to K$. Under these assumptions, a \emph{permutant} is defined as a finite set $H$ of maps from $Y$ to $X$ such that $ g\circ h\circ T\left(g^{-1}\right)\in H$ for any $h\in H$.
For each permutant $H$, it is possible to obtain a GENEO $F:\Phi\to\Psi$ by defining 
$F(\varphi)=\frac{1}{\mathrm{card\ }(H)} \sum_{h\in H} \varphi\circ h$~\cite{genperm2023}.
When the cardinality of $H$ is small, this construction method is particularly advantageous and fast since it does not require any expensive integration over the group $G$.

This paper explores the use of permutants to produce efficient GENEOs for distinguishing $r$-regular graphs (i.e., graphs where every node has a degree
equal to $r$) up to isomorphisms. In doing so, we aim to test the capabilities and flexibility of GENEOs, checking to what extent these operators can be used for data comparison.

Our experiments show that GENEOs offer a good compromise between efficiency and computational cost in comparing $r$-regular graphs, while their actions on data are easily interpretable.
In other words, using these operators could be a good general alternative to efficient but computationally expensive analysis methods. This further supports the idea that GENEOs could be a general-purpose approach to discriminative problems in Machine Learning when some structural information about data and observers is explicitly given.

\section*{Results}
The methodology outlined in the Methods section enabled us to generate a model for assessing graph non-isomorphism. Following a pruning process of an extensive initial pool of operators (sampled to ensure adequate exploration of operators built using permutants), the resulting architecture, as illustrated in Figure~\ref{fig:model}, is presented. The network was built adhering to established practices in machine learning, despite not possessing learnable parameters at this stage in a way such that simplicity and transparency are emphasized. Furthermore, the results shown in Figure~\ref{fig:time} and Table~\ref{tab:degree3} prove that, with respect to performance, our model is not falling behind other methods with comparable technical complexities.

\begin{figure}[ht]
\centering
\includegraphics[width=0.7\linewidth]{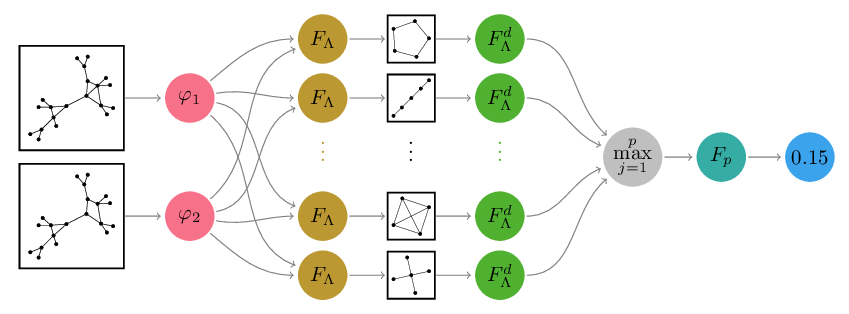}
\caption{Model architecture after the forward selection algorithm.}
\label{fig:model}
\end{figure}

Specifically, the contrast between $k$-WL tests brings into focus the indecisiveness of 1-WL and 2-WL, which are equivalent in terms of expressive power, when applied to $r$-regular graphs. On the other hand, 3-WL, albeit being more expressive than 1-WL, encounters significant computational challenges for large graphs. Due to the theoretical equivalence of expressivity between 1-WL and message passing GNNs, these observations can be extended to GNNs as well. In addition, GNNs, as the number of layers and parameters increase, quickly become opaque and lose explainability. Finally, considering the higher costs of $k$-WL, we expect a similar behavior for $k$-GNNs.

Regarding the other approaches, they are the most efficient in terms of computational costs but not as accurate, especially for $r$-regular graphs. This was foreseeable since they are fairly general methods without specificity for this class of graphs. Notably, NTX-COULD (which examines node degree distribution, triangle and maximal clique counts) may offer a reasonable balance between time efficiency and accuracy, as shown in Figure~\ref{fig:time}. However, unlike other methods, including those based on GENEOs, NTX-COULD does not scale linearly with respect to the graph dimension. Lastly, it is important to note that NTX-IS, which is the only exact algorithm under consideration, suffers significantly in the comparison due to its substantially longer execution times, even for moderately sized graphs (as illustrated in Table~\ref{tab:degree3}).

\section*{Discussion}
In this work, we proposed a methodology for testing the isomorphism of graphs through the application of Group Equivariant Non-Expansive Operators and permutants. Our approach is grounded in the framework of permutants, which enable the generation of efficient and effective GENEOs to map between the sets of functions representing unweighted and undirected graphs.

The Methods section provides a detailed exposition of the fundamental principles of permutants and their role in constructing GENEOs. Furthermore, we demonstrate how recently introduced permutants, namely Subgraph Permutants, can be employed to derive GENEOs for comparing pairs of graphs, thereby allowing for the construction of a network that considers various aspects relevant to the final decision.

In the Results section, we implement our proposed methodology for distinguishing between $r$-regular graphs and report numerical results that demonstrate its success in differentiating among various instances of these graphs. Notably, our method also exhibits competitive performance in terms of execution time, as evidenced by a time complexity that scales linearly with the dimensionality of the input. These findings on $r$-regular graphs are significant given their challenges to many existing state-of-the-art isomorphism tests, as illustrated in our comparative analysis.

Looking ahead, we plan to expand this analysis to accommodate weighted graphs and devise an isomorphism-based graph similarity computing network. This could potentially lead to more general graph comparison models that are less stringent in their requirement for isomorphism when assessing graph similarity. The extension requires careful consideration of appropriate permutant constructions for weighted graphs and corresponding GENEOs, which we intend to explore in future works.

\section*{Methods}
\subsection*{GENEOs and Permutants}

We will now briefly revisit the formal framework of GENEOs (Group Equivariant Non-Expansive Operators) and discuss permutants as a practical means to construct them. For a comprehensive understanding of GENEOs, please refer to~\cite{BeFrGiQu19}, and for further information on permutants, see~\cite{genperm2023,newmet2018,BoBoBrFrQu23}. We shall simply call ``permutant'' what is defined as ``generalized permutant'' in~ \cite{genperm2023}.

Let $X, Y$ be non-empty sets; let $\Phi$, $\Psi$ be subspaces, respectively, of
\[\mathbb{R}_b^X = \{\varphi:X \to \mathbb{R} \,|\, \varphi \ \ \text{is bounded}\} \ \ \ \ \ \mathbb{R}_b^Y = \{\psi:Y \to \mathbb{R} \,|\, \psi \ \ \text{is bounded}\}\]
both endowed with the $L^\infty$ norm. Let $G$ (resp. $K$) be a group of $\Phi$-preserving (resp. $\Psi$-preserving) bijections of $X$ (resp. $Y$) to itself. Finally, let $T:G \to K$ be a group homomorphism.

\begin{definition}[Group Equivariant Non-Expansive Operator]
    \label{def:geneo}
    A GENEO from the pair $(\Phi, G)$ to the pair $(\Psi, K)$, w.r.t. the homomorphism $T$, is a map $F:\Phi \to \Psi$ satisfying:
    \begin{enumerate}
        \item \textbf{(Equivariance)} $F(\varphi \circ g) = F(\varphi) \circ T(g)$ for every $\varphi \in \Phi$ and $g \in G$.
        \item \textbf{(Non-Expansiveness)} $\| F(\varphi_1) - F(\varphi_2)\|_{\infty} \le \|\varphi_1 - \varphi_2\|_{\infty}$ for every $\varphi_1, \varphi_2 \in \Phi$.
    \end{enumerate}
\end{definition}

Given two pairs $(\Phi, G)$ and $(\Psi, K)$ and a homomorphism $T\colon G \to K$, a (generalized) permutant can be defined as follows:

\begin{definition}
    \label{def:genperm}
    A finite subset $H \subseteq X^Y$ is a permutant for the homomorphism $T$ if either $H =\emptyset$ or $\{g\circ h \circ T(g^{-1})\ | \  h \in H\} \subseteq H$ for every $g \in G$.
\end{definition}

Permutants provide an easy way to construct GENEOs:

\begin{proposition}
    \label{prop:geneogenperm}
    If $H \subseteq X^Y$ is a non-empty permutant for $T \colon G \to K$ then the operator $F_H:\Phi \to \mathbb{R}^Y$  defined as 
    \begin{equation}
        \label{eqn:geneogenperm}
        F_H(\varphi) = \frac{1}{C} \sum_{h\in H} \varphi \circ h
    \end{equation}
    is a GENEO from $(\Phi, G)$ to $(\Psi, K)$, provided that $F_H(\Phi) \subseteq \Psi$ and that $C \ge |H|$.
\end{proposition}

Proposition~\ref{prop:geneogenperm} moreover shows that GENEOs obtained through permutants can also be efficiently computable: the smaller the permutant the faster the computation.

\subsection*{Graph setting and Subgraph Permutants}
In this section, we explore the application of GENEOs to the field of graphs. From hereon, our focus is on the set of undirected and unweighted simple graphs with a set of nodes $V_N = \{v_1, \ldots, v_N\}$, denoted as $\mathcal{G} = \{\Gamma = (V_\Gamma, E_\Gamma) \ |\ V_\Gamma = V_N\}$. Otherwise said, we consider all graphs of $\mathcal{G}$ as subgraphs of the complete graph $K_N$ on the vertex set $V_N$.

To proceed, let us define the domain as the set of index pairs $X = \{\{v_i, v_j\} \subset V_N \ |\  i\ne j \}$ (corresponding with the edge set of $K_N$), and the function space as $\Phi = \{\varphi \colon X \to \{0,1\}\}$. Each graph $\Gamma \in \mathcal{G}$ can be represented through a corresponding function $\varphi \in \Phi$ that models its adjacency matrix; we will denote this correspondence as $\Gamma \sim \varphi$. It is important to note that graphs with fewer than $N$ nodes can also be accommodated within our framework by adding as many 0-degree nodes as needed, without raising any problem for what is done in the following.

In the experimental section that follows, we will concentrate on addressing the Graph Isomorphism Problem. To this aim, we fix, as the equivariance group $G$, the group of all transformations $g\colon X \to X$ of the form $g(\{v_i, v_j\}) = \{v_{\sigma(i)}, v_{\sigma(j)}\}$, for some permutation $\sigma$ of the set $\{1, \dots, N\}$. Given two graphs $\Gamma_1$ and $\Gamma_2$ in $\Phi$, represented by functions $\varphi_1$ and $\varphi_2$ respectively, we shall deem them isomorphic if there exists a $g\in G$ such that $\varphi_1 = \varphi_2 \circ g$. To devise a test for checking whether two given graphs are isomorphic or not, this binary relationship serves as our ground truth. Effectively, it can be modeled by a well-known concept from the theory of Topological Data Analysis: the natural pseudo-distance~\cite{BeFrGiQu19}.
\begin{definition}
\label{def:natdist}
The natural pseudo-distance associated to a group $L\subseteq G$ is 
    \begin{equation}
    \label{eqn:natsdist}
        d_L(\varphi_1, \varphi_2) = \inf_{g \in L} \| \varphi_1 - \varphi_2 \circ g \|_{\infty}
    \end{equation}
\end{definition}
From this definition, due to the binary nature of the functions in $\Phi$, it easily follows that in our setting:
\begin{equation}
    \begin{aligned}
        \Gamma_1\text{ and }\Gamma_2\text{ are isomorphic }&\iff d_G(\varphi_1, \varphi_2) = 0\\
        \Gamma_1\text{ and }\Gamma_2\text{ are not isomorphic }&\iff d_G(\varphi_1, \varphi_2) = 1
    \end{aligned}
\end{equation}

We are now prepared to introduce Subgraph Permutants (SP). Given an auxiliary graph $\Lambda = (V_\Lambda, E_\Lambda)$, with $V_\Lambda = V_k =\{w_1, \ldots, w_k\}$, where $k \ll N$  (seen as a subgraph of the smaller complete graph $K_k$), we can represent it similarly as before in the functional space $\Psi_0 = \{ \psi_0\colon Y \to \{0,1\} \}$, where $Y = \{\{w_i, w_j\} \subset V_k \ | \ i\ne j \}\}$. Before proceeding with the definition, let us fix some notation that will be utilized subsequently. As previously stated, the group $G$ denotes the set of all graph transformations induced by a permutation of the $N$ nodes (i.e. self-isomorphisms of the complete graph $K_N$); in contrast, $G_\Gamma$ and $G_\Lambda$ refer to the groups of self-isomorphisms of graphs $\Gamma$ and $\Lambda$ respectively, i.e. those transformations of $K_N$ (resp. $K_k$) that map the subgraph $\Gamma$ (resp. $\Lambda$) back into itself. Furthermore, we define the functional space $\Psi = \{ \psi\colon Y \to [0,1] \}$ which consists of functions defined on $Y$ taking values within the closed interval $[0,1]$.

In this way, we can define:

\begin{definition}[Subgraph Permutant]
    \label{def:SGP}
    A Subgraph Permutant (SP) associated with the search of $\Lambda \sim \psi_0$ as a subgraph of $\Gamma \sim \varphi$ is the set of functions from $Y$ to $X$ of the form:
    \begin{equation}
        \label{eqn:SGP}
        H^\varphi_{\Lambda} = \{ h\colon Y \to X \ |\  (\varphi \circ h)(\{w_i, w_j\}) = 1 \iff \psi_0(\{w_i, w_j\}) = 1\quad\forall \{w_i, w_j\} \in Y, \quad h\text{ is injective}\}
    \end{equation}
\end{definition}

\begin{proposition}
    \label{prop:SGN}
    Every SP is a permutant for the trivial homomorphism $T\colon G_\Gamma \to \{Id_Y\}$.
\end{proposition}
\begin{proof}
    It must be shown that for every $g \in G_\Gamma$ and every $h \in H^\varphi_\Lambda$ it holds that $g\circ h \circ T(g^{-1}) \in H^\varphi_\Lambda$. But $T(g^{-1}) = Id_{Y}$ hence it is enough to show $g \circ h \in H^\varphi_\Lambda$. 
    Since $g$ is a self-isomorphism of $\Gamma$, the condition $(\varphi \circ g\circ h)(\{w_i, w_j\}) = 1$ is equivalent to the condition $(\varphi \circ h)(\{w_i, w_j\}) = 1$. This proves that $g \circ h \in H^\varphi_\Lambda$.
\end{proof}

We shall call `strict' a map $h$ respecting the condition of Def.~\ref{def:SGP}. As a map from $K_k$ to $K_N$, it maintains all values (0 and 1) from $\psi_0$ to $\varphi$. As an embedding of $\Lambda$ into $\Gamma$, it is such that $h(\Lambda)$ coincides with the subgraph of $\Gamma$ induced by its nodes. See Figure~\ref{fig:stream} for an embedding that is strict (on the right) and one that is not (on the left).

\begin{figure}[ht]
\centering
\begin{subfigure}{0.4\textwidth}
    \includegraphics[width=\textwidth]{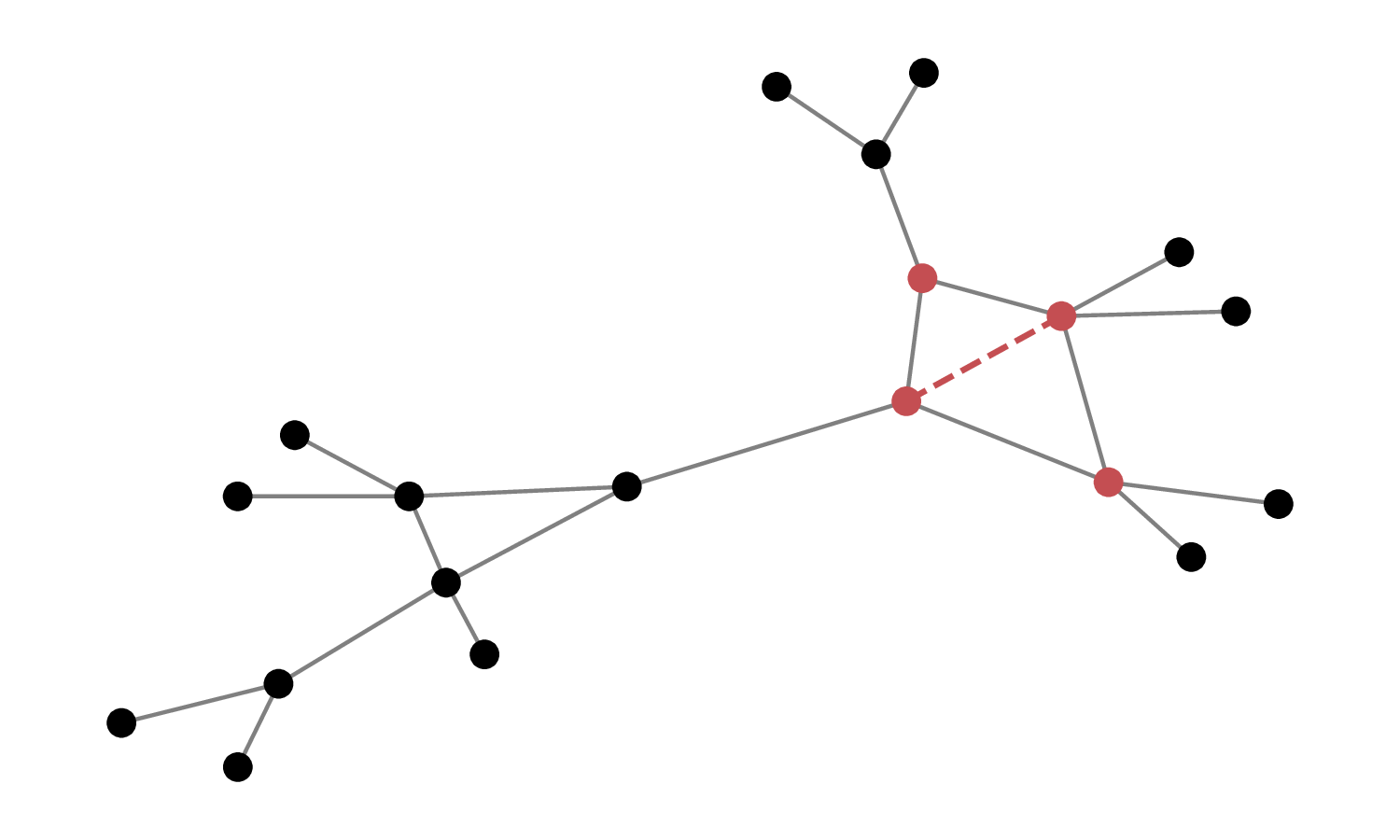}
    \caption{$\Gamma_1\quad H^{\varphi_1}_\Lambda = \emptyset$}
    \label{fig:first}
\end{subfigure}
\begin{subfigure}{0.4\textwidth}
    \includegraphics[width=\textwidth]{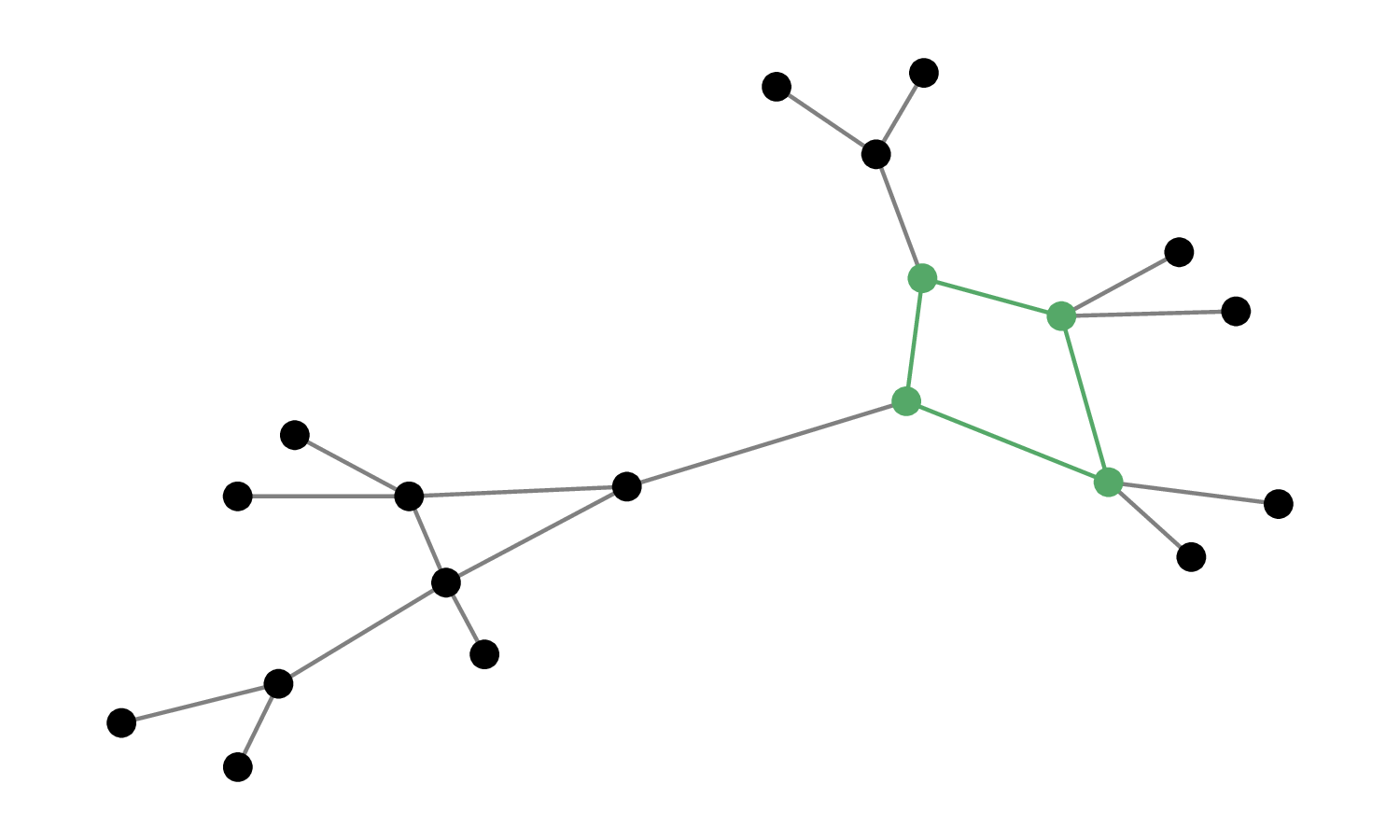}
    \caption{$\Gamma_2\quad |H^{\varphi_2}_\Lambda| = 8$}
    \label{fig:second}
\end{subfigure}
\caption{Considering $\Lambda$ as a 4-cycle graph, the set $H^{\varphi_1}_{\Lambda}$ is the empty SP since there is no strict embedding of $\Lambda$ to $\Gamma_1$. In fact, the red dashed edge joins a node pair whose $\varphi$-value is 1, while each pair of non-adjacent nodes in $\Lambda$ has vanishing $\psi_0$-value. On the other hand, there are 8 strict embeddings of $\Lambda$ into $\Gamma_2$ (i.e. as many as the self-isomorphisms of $\Lambda$), hence we obtain an SP of size $|H^{\varphi_2}_{\Lambda}| = 8$.}
\label{fig:stream}
\end{figure}

Indeed the cardinality of $H^{\varphi}_\Lambda$ is easily computable as $|H^{\varphi}_\Lambda| = n^\varphi_\Lambda |G_\Lambda|$ where $n^\varphi_\Lambda$ is the number of occurrences in $\Gamma\sim\varphi$ of subgraphs that are images of strict embeddings of $\Lambda$.

Once fixed $\widetilde{\Gamma} \sim \widetilde{\varphi}$, as well as $\Lambda \sim\psi_0$, we can define the GENEO $F_\Lambda^0$, which is represented as $
\displaystyle F_\Lambda^0(\varphi) = \frac{1}{|H^{\widetilde{\varphi}}_\Lambda|} \sum_{h \in H^{\widetilde{\varphi}}_{\Lambda}} \varphi \circ h
$, thanks to Proposition~\ref{prop:geneogenperm}. Unfortunately, this operator can only serve as a GENEO for the pairs $((\{ \widetilde{\varphi} \}, G_{\widetilde{\Gamma}}), (\Psi, \{ id_Y \}))$ with respect to the trivial homomorphism. The root cause of this limitation lies in the $\widetilde{\varphi}$-dependent normalization, which is tied to the fact that permutant $H^{\widetilde{\varphi}}_\Lambda$ is specific for $\widetilde{\Gamma}$. Consequently, we cannot utilize this definition for comparing other graphs $\Gamma$ up to transformations in $G$. However, we can define another GENEO that is invariant under the action of $G$ by maintaining a fixed $\Lambda$, and for each possible input $\Gamma \sim \varphi$, use the corresponding SP $H^\varphi_\Lambda$. It's important to note that if Proposition~\ref{prop:geneogenperm} fails to apply in this scenario, the subsequent new Proposition becomes applicable instead.

\begin{proposition}
\label{prop:SGPop}
The operator $F_\Lambda$ defined as
\begin{equation}
    \label{eqn:SGPop}
    F_\Lambda(\varphi) = \frac{1}{{\hat C}} \sum_{h\in H^\varphi_\Lambda} \varphi \circ h
\end{equation}
with the usual convention that a sum with no terms evaluates to 0, is a linear GENEO between the pairs $(\Phi, G),\, (\Psi, \{id_Y\})$ with respect to the trivial homomorphism $T:G\to \{id_Y\}$, where ${\hat C} = n_\Lambda |G_\Lambda|$, and $n_\Lambda$ is the number of occurrences of $\Lambda$ in a complete graph with $N$ nodes.
\end{proposition}

\begin{proof}
We need to show equivariance and non-expansiveness of $F_\Lambda$.
\begin{itemize}
\item \textbf{Equivariance:} we note that, since $h \in H^\varphi_\Lambda$ is injective, for every isomorphism $g \in G$ the map $h \mapsto g\circ h$ is a bijection between $H^{\varphi \circ g}_\Lambda$ and $H^{\varphi}_\Lambda$, that is $h \in H^{\varphi \circ g}_\Lambda \iff g \circ h \in H^{\varphi}_\Lambda$, hence:
\[
\begin{aligned}
F_\Lambda(\varphi \circ g) &= \frac{1}{{\hat C}} \sum_{h \in H^{\varphi \circ g}_\Lambda} \varphi \circ g \circ h\\
 &= \frac{1}{{\hat C}} \sum_{h' \in H^{\varphi}_\Lambda} \varphi \circ h'\\ &= F_\Lambda(\varphi)= F_\Lambda(\varphi) \circ Id_Y = F_\Lambda(\varphi) \circ T(g) 
\end{aligned}
\]
\item \textbf{Non-Expansiveness:} 
If $\varphi_1 \ne \varphi_2$ (this implies that the two functions must differ at least on a pair of vertices) then, since $\varphi_1,\varphi_2$ take values in the set $\{0,1\}$, $\|\varphi_1 - \varphi_2\|_{\infty} = 1$. Hence given $e = \{w_i, w_j\} \in Y$:
	\[
		\begin{aligned}
			|F_\Lambda(\varphi_1)(e) - F_\Lambda(\varphi_2)(e)| &=\left|\frac{1}{{\hat C}} \sum_{h \in H^{\varphi_1}_\Lambda} (\varphi_1 \circ h)(e) - \frac{1}{{\hat C}} \sum_{h \in H^{\varphi_2}_\Lambda} (\varphi_2\circ h)(e)\right|\\
			&=\frac{1}{{\hat C}} \biggl|\left|H^{\varphi_1}_\Lambda\right| - \left|H^{\varphi_2}_\Lambda\right|\biggr| \\
			&\le \frac{\max\left(\left|H^{\varphi_1}_\Lambda\right|, \left|H^{\varphi_2}_\Lambda\right|\right)}{{\hat C}} \\
            &= \frac{\max\left\{n^{\varphi_1}_\Lambda |G_\Lambda|,n^{\varphi_2}_\Lambda |G_\Lambda|\right\}}{n_\Lambda |G_\Lambda|}\\
            &= \frac{\max\left\{n^{\varphi_1}_\Lambda,n^{\varphi_2}_\Lambda \right\}}{n_\Lambda}\\
			&\le 1 = \|\varphi_1 - \varphi_2\|_{\infty}\\
		\end{aligned}
	\]
By taking the supremum on $e$ in the left hand side we get the inequality $\| F_\Lambda(\varphi_1) - F_\Lambda(\varphi_2)\|_{\infty} \le \|\varphi_1 - \varphi_2\|_{\infty}$.
Otherwise, if $\varphi_1 = \varphi_2$ we get that $F_\Lambda(\varphi_1) = F_\Lambda(\varphi_2)$, thus $\| F_\Lambda(\varphi_1) - F_\Lambda(\varphi_2)\|_{\infty} = 0 =\|\varphi_1 - \varphi_2\|_{\infty}$.
\end{itemize}

\end{proof}

The normalizing constant ${\hat C} = n_\Lambda |G_\Lambda|$ is not easily computable for a generic $\Lambda$, anyway a simple upper bound can be found by considering the number of injective functions from $k$ elements to $N$ elements, literally ${\hat C} \le \frac{N!}{(N-k)!}$. Thus, by using the upper bound instead of ${\hat C}$ in Proposition~\ref{prop:SGPop}, non-expansiveness is still guaranteed.

Finally, we introduce the last operator that will be used as the basic unit for comparing graphs:
\begin{definition}
\label{def:diffop}
Given $\Lambda$ as before we can define an operator on pairs of graphs ${F^d_\Lambda} \colon \Phi \times \Phi \to \Psi$ in the following way:
    \begin{equation}
    \label{eqn:diffop}
        {F^d_\Lambda}(\varphi_1, \varphi_2) = \frac{1}{2}|F_{\Lambda}(\varphi_1) - F_{\Lambda}(\varphi_2)|
    \end{equation}
\end{definition}

\begin{proposition}
    The operator defined in Equation~\ref{eqn:diffop} is a GENEO from $(\Phi \times \Phi, G\times G)$ to $(\Psi, \{id_Y\})$ with respect to the trivial homomorphism $T:G\times G\to\{id_Y\}$.
\end{proposition}
\begin{proof}
    We need to show equivariance and non-expansiveness of ${F^d_\Lambda}$:

    \begin{itemize}
    \item \textbf{Equivariance:}
    $$\begin{aligned}
    &{F^d_\Lambda}((\varphi_1, \varphi_2) \circ (g_1, g_2)) =\\  &= \frac{1}{2}|F_{\Lambda}(\varphi_1 \circ g_1) - F_{\Lambda}(\varphi_2 \circ g_2)|\\ &= \frac{1}{2}|F_{\Lambda}(\varphi_1) - F_{\Lambda}(\varphi_2)|\\ &= {F^d_\Lambda}(\varphi_1, \varphi_2)
    = {F^d_\Lambda}(\varphi_1, \varphi_2) \circ id_Y 
    = {F^d_\Lambda}(\varphi_1, \varphi_2) \circ T((g_1,g_2))
    \end{aligned}$$
    
     as $F_{\Lambda}$ is invariant with respect to any $g \in G$.
    
    \item \textbf{Non-Expansiveness:} 

    Let us start by noting that for any $\varphi\in\Phi$, 
    given the fact that the functions in $\Phi$ have only 0 and 1 as possible values, 
$$
F_\Lambda(\varphi)(\{w_i, w_j\}) = \frac{1}{{\hat C}} \sum_{h \in H^\varphi_\Lambda} (\varphi \circ h)(\{w_i, w_j\}) = 
\begin{cases}
    \frac{1}{{\hat C}} \sum_{h \in H^\varphi_\Lambda} 1 = |H^\varphi_\Lambda|/{\hat C}&\text{  if  }\psi_0(\{w_i, w_j\}) = 1 \\
   \frac{1}{{\hat C}} \sum_{h \in H^\varphi_\Lambda} 0 = 0  &\text{  if  }\psi_0(\{w_i, w_j\}) = 0 \\
\end{cases}.
$$
It follows that the values taken by the function 
$F_\Lambda(\varphi)$ belong to the set 
$\left\{0,|H^\varphi_\Lambda|/{\hat C}\right\}$.

Therefore, for any $e =\{w_i,w_j\}\in Y$
    
    $$\begin{aligned}  
    &|{F^d_\Lambda}((\varphi_1, \varphi_2))(e) - {F^d_\Lambda}((\varphi_3, \varphi_4))(e)|\\
    &= \frac{1}{2}\Bigl||F_{\Lambda}(\varphi_1)(e) - F_{\Lambda}(\varphi_2)(e)|-|F_{\Lambda}(\varphi_3)(e) - F_{\Lambda}(\varphi_4)(e)|\Bigr|\\
     &\le \frac{1}{2}\Bigl|F_{\Lambda}(\varphi_1)(e) - F_{\Lambda}(\varphi_2)(e)-F_{\Lambda}(\varphi_3)(e) + F_{\Lambda}(\varphi_4)(e)\Bigr|\\
    &= \frac{1}{2}\Bigl|\big(F_{\Lambda}(\varphi_1)(e) - F_{\Lambda}(\varphi_3)(e)\big)-\big(F_{\Lambda}(\varphi_2)(e) - F_{\Lambda}(\varphi_4)(e)\big)\Bigr|\\
    &\le \frac{1}{2}\big(\big|F_{\Lambda}(\varphi_1)(e) - F_{\Lambda}(\varphi_3)(e)\big| + \big|F_{\Lambda}(\varphi_2)(e) - F_{\Lambda}(\varphi_4)(e)\big|\big)\\
    &\le \frac{1}{2}\big(\|F_{\Lambda}(\varphi_1) - F_{\Lambda}(\varphi_3)\|_{\infty} + \|F_{\Lambda}(\varphi_2) - F_{\Lambda}(\varphi_4)\|_{\infty}\big)\\
    &\le \frac{1}{2}\left(\|\varphi_1 - \varphi_3\|_{\infty}+ \|\varphi_2 - \varphi_4\|_{\infty}\right)\\
    &\le \max\left(\|\varphi_1 - \varphi_3\|_{\infty}, \|\varphi_2 - \varphi_4\|_{\infty}\right)\\
    &= \left\|\left(\varphi_1,\varphi_2\right)-\left(\varphi_3,\varphi_4\right)\right\|_{\infty}
    \end{aligned}$$
   where the last symbol $\|\cdot\|_\infty$ denotes the max-norm in $\Phi\times\Phi$.
    Hence taking the supremum over $e =\{w_i,w_j\} \in Y$:
    \[
    \|{F^d_\Lambda}((\varphi_1, \varphi_2)) - {F^d_\Lambda}((\varphi_3, \varphi_4))\|_{\infty} \le \left\|\left(\varphi_1,\varphi_2\right)-\left(\varphi_3,\varphi_4\right)\right\|_{\infty}.
    \]
    \end{itemize}
\end{proof}

The value $\big| |H^{\varphi_1}_\Lambda|/{\hat C}-|H^{\varphi_2}_\Lambda|/{\hat C}\big|$ taken by function ${F^d_\Lambda}((\varphi_1,\varphi_2))$ at every edge of $\Lambda$ can be interpreted as a ``non-isomorphism score'' for the pair $(\varphi_1,\varphi_2)$, with respect to subgraph $\Lambda$. We will label this score as $\chi_\Lambda(\varphi_1, \varphi_2) = \|{F^d_\Lambda}((\varphi_1,\varphi_2))\|_{\infty} \in [0,1]$.

\subsection*{Subgraph GENEO Network}

To begin, let us consider a collection of $p \ge 2$ distinct subgraphs, denoted as $\Lambda_1,\dots,\Lambda_p$. By stacking the information provided by the GENEOs $F_{\Lambda_j}$ for each of these subgraphs, we can generate an isomorphism-invariant $p$-dimensional vector.
$$
\chi_{\{1,\dots,p\}}((\varphi_1, \varphi_2)) = \begin{pmatrix} \chi_{\Lambda_1}((\varphi_1, \varphi_2)) \\  \chi_{\Lambda_2}((\varphi_1, \varphi_2))\\\vdots\\\chi_{\Lambda_p}((\varphi_1, \varphi_2)) \end{pmatrix} 
$$

Subsequently, we can find the maximum of our $p$ GENEOs that, according to the rules for algebraic operations between GENEOs outlined in~\cite{remalg2017}, still results in a valid GENEO. 

\begin{proposition}
 The map $F_{\{1,\dots,p\}}$ taking each pair $(\varphi_1,\varphi_2)$ to the number $\max\left(\chi_{\Lambda_1}((\varphi_1, \varphi_2)),\ldots,\chi_{\Lambda_p}((\varphi_1, \varphi_2))\right)$ is a GENEO from $(\Phi \times \Phi, G\times G)$ to $(\R, \{id_S\})$ with respect to the trivial homomorphism $T':G\times G\to \{id_S\}$ (where $\R$ is identified with the set of real-valued functions defined on a fixed singleton $S$).
\end{proposition}

\begin{proof}
    We consider the map $R$ taking each function $\psi\in\Psi$ to the number $\|\psi\|_\infty$. The inequality $\big|
    \|\psi_1\|_\infty-\|\psi_2\|_\infty
    \big|\le\|\psi_1-\psi_2\|_\infty$ shows that $R$ is non-expansive. $R$ is trivially invariant under the action of $\{id_Y\}$, so it is a GENEO from $(\Psi, \{id_Y\})$ to $(\R, \{id_S\})$ with respect to the trivial homomorphism $T'':\{id_Y\}\to\{id_S\}$. Each $\chi_{\Lambda_i}$ is then the composition of the GENEO $F_{\Lambda_i}^d$ from $(\Phi \times \Phi, G\times G)$ to $(\Psi, \{id_Y\})$ with the GENEO $R$, so a GENEO itself from $(\Phi \times \Phi, G\times G)$ to $(\R, \{id_S\})$. Finally, a simple extension of the proof of Prop. 2 in~\cite{remalg2017} shows that the maximum of a set of GENEOs is a GENEO with the same domain and range.
\end{proof}

This GENEO will combine the observations of the first level GENEOs $F_{\Lambda_j}$ with a vetoing perspective. The resulting aggregated operator $F_{\{1,\dots,p\}}$ (denoted $F_p$ for brevity when considering all indices up to $p$), provides a more refined non-isomorphism score according to the following theorem.

\begin{theorem}
    \label{thm:iso}
    Given two graphs $\Gamma_1 \sim \varphi_1$ and $\Gamma_2 \sim \varphi_2$, for every choice of $p\ge 1$ and $\Lambda_1, \dots, \Lambda_p$ the followings hold:
    \begin{enumerate}
        \item If $\Gamma_1$ and $\Gamma_2$ are isomorphic then $\chi_{\{1,\dots,p\}}((\varphi_1, \varphi_2)) = \mathbf{0}$ or equivalently $F_{\{1,\dots,p\}}((\varphi_1, \varphi_2)) = 0$.
        \item If $\chi_{\{1,\dots,p\}}((\varphi_1, \varphi_2)) \ne \mathbf{0}$ or equivalently $F_{\{1,\dots,p\}}((\varphi_1, \varphi_2)) \ne 0$ then $\Gamma_1$ and $\Gamma_2$ are not isomorphic.
    \end{enumerate}
\end{theorem}
\begin{proof}
Straightforward.
\end{proof}
    
Given Theorem~\ref{thm:iso}, we can summarize the descriptions thus far into a model, the architecture of which is illustrated in Figure~\ref{fig:model}, possibly allowing the model to consist of more than three operators.

\subsection*{Experimental setup}

\begin{figure}[ht]
\centering
\includegraphics[width=0.9\linewidth]{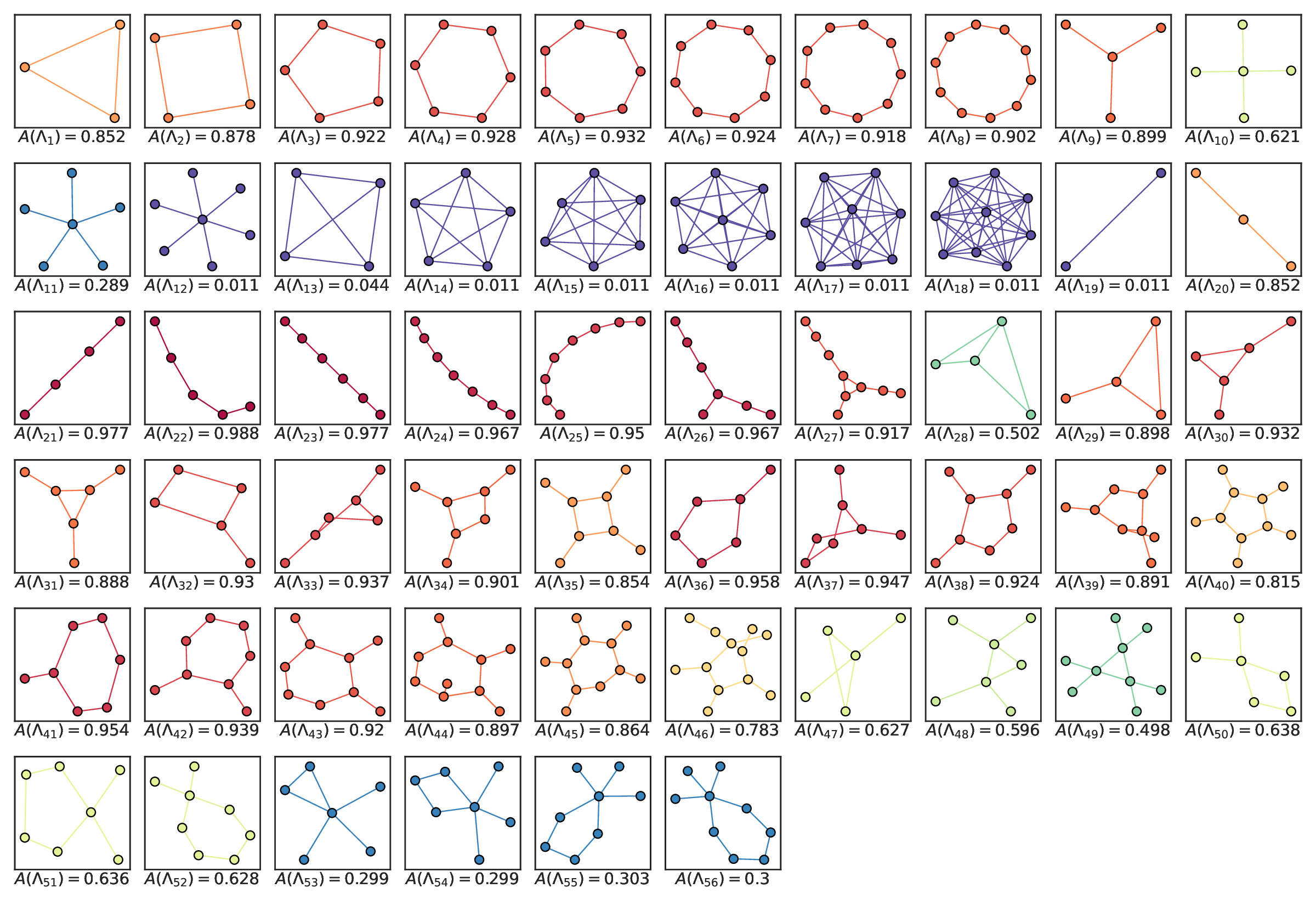}
\caption{The representations of the $\Lambda_j$ initially considered for analysis can be categorized into a limited number of classes. There are cycles, represented as $\Lambda_j$ for $j$ ranging from 1 to 8; stars, ranging from 9 to 12; complete graphs, ranging from 13 to 18; paths, ranging from 19 to 25; rigid graphs with no non-trivial self-isomorphism at 26 and 27 and cycles augmented with stars ranging from 29 to 56. Each $\Lambda_j$ also reports the accuracy of the operator $F_{\{j\}}$ using both colors and a numeric value.}
\label{fig:lambdas}
\end{figure}

To evaluate and assess the methodology discussed in the previous sections, we designed an experimental setup as outlined below. We chose to focus on the family of $r$-regular graphs ($r\ge 2$): a graph $\Gamma$ is defined as $r$-regular if every node has degree equal to $r$. It is well-known that $r$-regular graphs present challenges when testing for isomorphism, as the degree distribution often plays a crucial role in many algorithms. However, for $r$-regular graphs with identical degrees and equal numbers of nodes, the degree distribution does not provide any valuable discriminating information. It is equally well known that $r$-regular graphs are one of those classes of graphs for which polynomial algorithms exist. In the works of Luks et al.~\cite{boundval1982, n3logn1987}, a general bound of kind $\mathcal{O}(n^{\text{poly}(r)})$ was stated, bound that was specialized to $\mathcal{O}(n^3 \log n)$ for $3$-regular graphs. However, implementing these algorithms efficiently remains an open research question, making them primarily of theoretical interest. As a result, there is ongoing research into developing algorithms~\cite{pract2014, vf22018} that do not meet the polynomial time guarantee but could offer advantages even when dealing with graphs for which polynomial algorithms are known such as $r$-regular graphs.

Furthermore, it is well-known that $r$-regular graphs, lacking node weights or features to break symmetry, are indistinguishable when tested using the Weisfeiler-Lehman coloring test (1-WL)~\cite{satosurvey2020}. This inability to distinguish between $r$-regular graphs has been proven~\cite{wlgnneq2019, powgnn2019} to be inherited by message passing Graph Neural Networks (GNNs)~\cite{gnn2005}, which perform neighbor aggregation in a manner similar to 1-WL. To address such issues, more complex variants of 1-WL and GNNs, specifically $k$-WL~\cite{satosurvey2020} and $k$-GNNs~\cite{kgnn2019}, have been proposed. These advanced methods offer greater expressivity in terms of discrimination, but they also entail increased computational costs due to their consideration of all tuples of nodes instead of just neighbors.

Our first experiment focused on random samples of $r$-regular graphs with $r \in \{3, 4, 5\}$. We analyzed these graphs one $r$ at a time and then in combination. To construct a subgraph GENEO network, we considered 56 distinct subgraphs $\Lambda_j$, which are shown in Figure~\ref{fig:lambdas}. For each $r$ and every value of $N$ between 8 and 100 (with a step size of 2), a random sample of 10 $r$-regular graphs was generated. Consequently, for each $N$, we tested 45 pairs of graphs. In total, there were $2115$ unique pairs for each $r$ and $6345$ pairs that combined data from different $r$. As a starting point, we evaluated the discriminating capacity of every GENEO linked to a single $\Lambda_j$ by calculating its classification accuracy according to:

\begin{equation}
\label{eqn:acc}
A\Bigl(F_{\{j\}}\Bigr) = \frac{1}{m} \sum_{l=1}^m \mathbbm{1}\Bigl(\sgn(F_{\{j\}}((\varphi_1, \varphi_2)_l)) = d_G((\varphi_1, \varphi_2)_l)\Bigr)
\end{equation}

In Eqn.~\ref{eqn:acc} a dataset of $m$ graph pairs and associated ground truths $\bigl((\varphi_1, \varphi_2)_l, d_G((\varphi_1, \varphi_2)_l)\bigr)_{l=1}^m$ is used for the estimation, the formula simply computes the proportion of examples such that the sign of the non-isomorphism score given by operator $F_{\{j\}}$ coincides with the ground truth. Indeed the function $\mathbbm{1}\Bigl(\sgn(F_{\{j\}}((\varphi_1, \varphi_2)_l)) = d_G((\varphi_1, \varphi_2)_l)\Bigr)$ evaluates to 1 when the sign of the score and the ground truth are equal and 0 otherwise.

Based on Figure~\ref{fig:lambdas}, which also displays accuracy values, it becomes apparent that path subgraphs generally excel when considered individually. Subsequently, our focus shifted to combining individual GENEOs within a model architecture such as the one depicted in Figure~\ref{fig:model}. To narrow down our selection and maintain satisfactory performance on the given data, we employed the forward selection technique outlined in Algorithm~\ref{alg:forw}.
\begin{algorithm}
    \caption{Forward Selection}\label{alg:forw}
    \KwData{$\Lambda_1, \dots, \Lambda_p$, $((\varphi_1,\varphi_2)_k)_{k=1}^m$}
    \KwResult{$S \subseteq \{\Lambda_1, \dots, \Lambda_p\}$}
    $l \gets 1$;
    $i_l \gets \underset{t\in\{1,\dots,p\}}{\arg\max}\, A\Bigl(F_{\{j\}}\Bigr)$\;
    \While{$\underset{t\in\{1,\dots,p\} \setminus \{i_1,\dots,i_l\}}{\max} A\Bigl(F_{\{i_1,\dots,i_l\} \cup \{t\}}\Bigr) > A\Bigl(F_{\{i_1,\dots,i_l\}}\Bigr)$}{
    $i_{l+1} \gets \underset{t\in\{1,\dots,p\} \setminus \{i_1,\dots,i_l\}}{\arg\max}\, A\Bigl(F_{\{i_1,\dots,i_l\} \cup \{t\}}\Bigr)$\;
    }
    $S = \{\Lambda_{i_1}, \dots, \Lambda_{i_l}\}$\;
\end{algorithm}

Upon executing the forward selection process, the algorithm halts at the third iteration, selecting operators $F_{\Lambda_{22}}, F_{\Lambda_{21}}$, and $F_{\Lambda_4}$. The resulting accuracy is $1.00$, however, it is important to note that this high accuracy is clearly overestimated due to the limited sample size under examination. Nonetheless, these results prove functional for the subsequent analysis.

\subsection*{Comparison}

In our second experiment, we contrasted the performance of the model derived by forward selection against a set of user-friendly state-of-the-art methods. These methods include 1-WL (NetworkX~\cite{networkx2008} implementation), 2-WL and 3-WL (our custom implementation), NTX-FASTER, NTX-FAST, NTX-COULD, and NTX-IS (respectively \verb|faster_could_be_isomorphic|, \verb|fast_could_be_isomorphic|, \verb|could_be_isomorphic| and \verb|is_isomorphic| from NetworkX library, in particular NTX-IS uses the \verb|VF2++| exact algorithm~\cite{vf22018}). The algorithms were assessed on a dataset of randomly generated non-isomorphic $r$-regular pairs, with $r$ taking values in the set $\{3, 4, 5\}$ and $N$ in $\{100, 500, 1000, 5000, 10000\}$. Specifically, for each combination of $r$ and $N$, we generated 500 unique pairs (distinct from those utilized in prior sections) of non-isomorphic graphs. The results regarding execution times (these and previous computations were executed on a laptop equipped with an eight-core Intel Core i7-6700HQ CPU) are visualized in Figure~\ref{fig:time}, while detailed information on both execution times and accuracies (only for $r=3$) can be found in Table~\ref{tab:degree3}.

\begin{figure}[ht]
\centering
\includegraphics[width=0.9\linewidth]{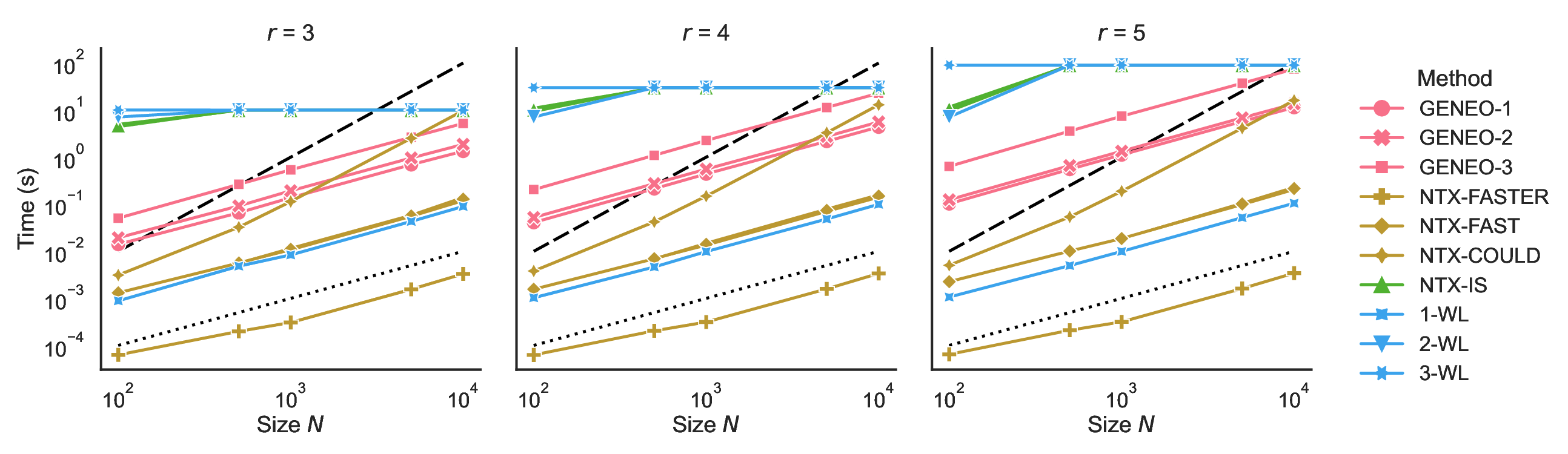}
\caption{The average execution time in function of the graph Size $N$ for the considered values of degree $r$: The GENEO-$t$ model (with $t$ being 1, 2 or 3) represents the method that aggregates the first $t$ operators identified through the forward selection process. The dotted and dashed lines represent respectively a linear and a quadratic growth.}
\label{fig:time}
\end{figure}
\begin{table}[htp]
\caption{\label{tab:degree3} The average execution times and corresponding accuracies for $r=3$ are presented below. Bold values indicate the most accurate method(s) in terms of accuracy. For every value of $r$, a maximum execution time was predetermined to prevent overly lengthy computations from occurring. Specifically, the timeout was set at 10 seconds for $r=3$, 30 seconds for $r=4$, and 90 seconds for $r=5$. It is important to note that the accuracy values are influenced by this time restriction. If a method fails to conclude within the designated timeframe, its result is automatically set to null, which counts as an incorrect classification of a non-isomorphic pair. As a result, methods like NTX-IS, which is an exact algorithm, may exhibit accuracy scores below 1.}
\centering
\small
\begin{tabular}{lcccccccccc}
\toprule
\multirow{3}*{Method} & \multicolumn{10}{c}{$N$} \\
\cmidrule(lr){2-11}
& \multicolumn{2}{c}{100} & \multicolumn{2}{c}{500} & \multicolumn{2}{c}{1000} & \multicolumn{2}{c}{5000} & \multicolumn{2}{c}{10000} \\
\cmidrule(lr){2-3} \cmidrule(lr){4-5} \cmidrule(lr){6-7} \cmidrule(lr){8-9} \cmidrule(lr){10-11}
& \small{Time} & \small{Accuracy} & \small{Time} & \small{Accuracy} & \small{Time} & \small{Accuracy} & \small{Time} & \small{Accuracy} & \small{Time} & \small{Accuracy} \\
\midrule
GENEO-1	    & 0.014	& 0.978	& 0.065	& 0.987	& 0.136	& 0.987 & 0.685 & 0.982 & 1.327 & 0.982\\
GENEO-2	    & 0.019	& 0.992	& 0.091	& 0.997	& 0.191	& 0.992 & 0.952 & 0.995 & 1.851 & 0.993\\
GENEO-3	    & 0.050	& \textbf{1.000	}& 0.261 & \textbf{1.000} & 0.532	& \textbf{1.000} & 2.626 & \textbf{1.000} & 5.167 & \textbf{1.000}\\
NTX-FASTER	& 0.000	& 0.000	& 0.000	& 0.000 & 0.000 & 0.000 & 0.002 & 0.000 & 0.003 & 0.000\\
NTX-FAST	& 0.001	& 0.728	& 0.006	& 0.775 & 0.011	& 0.755 & 0.057 & 0.780 & 0.129 & 0.758\\
NTX-COULD	& 0.003	& 0.728	& 0.033	& 0.775 & 0.114	& 0.755 & 2.520 & 0.780 & 9.737 & 0.758\\
NTX-IS	    & 5.021	& \textbf{0.667}& 9.959	& \textbf{0.005}  & 9.984	& \textbf{0.002} & 10.000 & \textbf{0.000}	 & 10.000 & \textbf{0.000}\\
1-WL	    & 0.001	& 0.000	& 0.005	& 0.000 & 0.008	& 0.000 & 0.043 & 0.000 & 0.090 & 0.000\\
2-WL	    & 7.091	& 0.000	& 10.000 & 0.000 & 10.000 & 0.000 & 10.000 & 0.000 & 10.000	& 0.000\\
3-WL	    & 10.000 & 0.000 & 10.000 & 0.000 & 10.000 & 0.000 & 10.000	& 0.000	& 10.000 & 0.000\\
\bottomrule
\end{tabular}
\end{table}

\section*{Acknowledgements}

P.F. carried out his research in the framework of the CNIT WiLab National Laboratory and the WiLab-Huawei Joint Innovation Center and was partially supported by INdAM-GNSAGA and the COST Action CaLISTA. G.B. was partially supported by INdAM-GNAMPA.

\section*{Competing interests}

The authors declare no competing interests.

\section*{Author contributions}

G.B., M.F., and P.F. developed the mathematical model and wrote the text of the paper. G.B. implemented the algorithms and ran the experiments. All authors reviewed the manuscript.

\section*{Data and Code Availability}
All data and codes used in this research are openly accessible on GitHub at \href{https://github.com/jb-sharp/spgeneos}{https://github.com/jb-sharp/spgeneos}

\printbibliography

\end{document}